\documentclass{article}
\usepackage{ijcai18}
\usepackage{times}
\usepackage{xcolor}
\usepackage{soul}
\usepackage[utf8]{inputenc}
\usepackage[small]{caption}
\usepackage{amsmath}
\usepackage{amssymb}
\usepackage{epsfig}
\usepackage{algorithm}
\usepackage[noend]{algorithmic}
\usepackage[hyperindex,breaklinks]{hyperref}
\newtheorem{definition}{Definition}

\newtheorem{theorem}{Theorem}

\newcommand\shrink[1]{}

\def\n(#1){\bar{#1}}

\def\pr{{\it Pr}}

\def\X{{\bf X}}
\def\x{{\bf x}}
\def\Y{{\bf Y}}
\def\y{{\bf y}}
\def\Z{{\bf Z}}
\def\z{{\bf z}}
\def\eql(#1,#2){{#1\!\!=\!#2}}

\def\eql(#1,#2){{#1\!=\!#2}}
\newcommand\name[1]{\ensuremath{\mathsf{#1}}}

\def\eproof{{\hfill$\Box$}}
\newenvironment{proof}[1][Proof]
			   {\begin{trivlist}\item[\hskip \labelsep {\bfseries #1}]}
			   {~\eproof \end{trivlist}}
\newenvironment{proofendline}[1][Proof]
               {\begin{trivlist}\item[\hskip \labelsep {\bfseries #1}]}
               {\end{trivlist}}

\def\clap#1{\hbox to 0pt{\hss#1\hss}}

\newcommand\Amain[1]{
\textbf{main:}
#1
\vspace{1mm}
}

\newcommand\Ainput[1]{
\vspace{1mm}
\textbf{input:} #1
}

\newcommand\Aoutput[1]{
\vspace{1mm}
\textbf{output:} #1
\vspace{1mm}
}

\def\mincard{\name{minimize}}
\def\conjoin{\name{conjoin}}

\def\posve{+}
\def\negve{-}

\title{A Symbolic Approach to Explaining Bayesian Network Classifiers}

\author{
Andy Shih \and
Arthur Choi \and
Adnan Darwiche
\\ 
Computer Science Department \\
University of California, Los Angeles\\
\texttt{\{andyshih,aychoi,darwiche\}@cs.ucla.edu}
}

\begin{document}

\maketitle

\begin{abstract}
We propose an approach for explaining Bayesian network classifiers, which is based on compiling such classifiers into decision functions that have a tractable and symbolic form. We introduce two types of explanations for why a classifier may have classified an instance positively or negatively and suggest algorithms for computing these explanations. The first type of explanation identifies a minimal set of the currently active features that is responsible for the current classification, while the second type of explanation identifies a minimal set of features whose current state (active or not) is sufficient for the classification. We consider in particular the compilation of Naive and Latent-Tree Bayesian network classifiers into Ordered Decision Diagrams (ODDs), providing a context for evaluating our proposal using case studies and experiments based on classifiers from the literature.
\end{abstract}

\section{Introduction}

Recent progress in artificial intelligence and the increased deployment of AI systems have led to highlighting the need for \emph{explaining} the decisions made by such systems, particularly classifiers; see, e.g., \cite{lime:kdd16,ElenbergDFK17,LundbergL17,anchors:aaai18}.\footnote{It is now recognized that
opacity, or lack of explainability is ``one of the biggest obstacles to widespread adoption of artificial intelligence'' (The Wall Street Journal, August 10, 2017).}
For example, one may want to explain {\em why} a classifier decided to turn down a loan application, or rejected an applicant for an academic program,
or recommended surgery for a patient. Answering such {\em why?} questions is particularly central to assigning blame and responsibility, 
which lies at the heart of legal systems and may be required in certain contexts.\footnote{See, for example, the EU general data protection regulation, which has a provision relating to 
explainability, \url{https://www.privacy-regulation.eu/en/r71.htm}.}

In this paper, we propose a \emph{symbolic} approach to explaining Bayesian network classifiers, which is based on the following observation.
Consider a classifier that labels a given instance either positively or negatively based on a number of discrete features. Regardless of how this classifier is implemented,
e.g., using a Bayesian network, it does specify a symbolic function that maps features into a yes/no decision (yes for a positive instance).
We refer to this function as the classifier's \emph{decision function} since it unambiguously describes the classifier's behavior, independently of how the classifier
is implemented.  Our goal is then to obtain a symbolic and tractable representation of this decision function, to enable symbolic and efficient reasoning about its
behavior, including the generation of explanations for its decisions. In fact, \cite{chanUAI03} showed how to compile the decision functions of naive Bayes classifiers into a specific symbolic and tractable representation, known as Ordered Decision Diagrams (ODDs). This representation extends Ordered Binary Decision Diagrams (OBDDs) to use multi-valued variables (discrete features), 
while maintaining the tractability and properties of OBDD~\cite{Bryant86,MeinelT98,Wegener00}.

We show in this paper how compiling decision functions into ODDs can facilitate the efficient explanation of classifiers
and propose two types of explanations for this purpose.

The first class of explanations we consider are \emph{minimum-cardinality explanations}.  To motivate these explanations, consider a classifier that has diagnosed a patient with some disease based on some observed test results, some of which were positive and others negative.  Some of the positive test results may not be necessary for the classifier's decision: the decision would remain intact if these test results were negative.
A minimum-cardinality explanation then tells us which of the positive test results are the culprits for the classifier's decision, i.e., a minimal subset of the positive test results that is sufficient for the current decision.

The second class of explanations we consider are \emph{prime-implicant explanations}. These explanations answer the following question:
what is the smallest subset of features that renders the remaining features irrelevant to the current decision?  
In other words, which subset of features---when fixed---would allow us to arbitrarily toggle the values of other features, while maintaining the classifier's decision?

This paper is structured as follows.
In Section~\ref{sec:classifiers}, we review the compilation of naive Bayes classifiers into ODDs, and propose a new algorithm for compiling latent-tree classifiers into ODDs.  
In Section~\ref{sec:mincard}, we introduce minimum-cardinality explanations, propose an algorithm for computing them, and provide a case study on a real-world classifier.  
In Section~\ref{sec:primes}, we do the same for prime-implicant explanations.  In Section~\ref{sec:more-monotone}, we discuss the relationship between the two types of explanations and 
show that they coincide for monotone classifiers.
We then follow by a discussion of related work in Section~\ref{sec:related} and finally close in Section~\ref{sec:conclusion}. 

\section{Compiling Bayesian Network Classifiers} \label{sec:classifiers}

\begin{figure}[t]
  \centering
  \includegraphics[width=.8\linewidth,clip=true,angle=0]{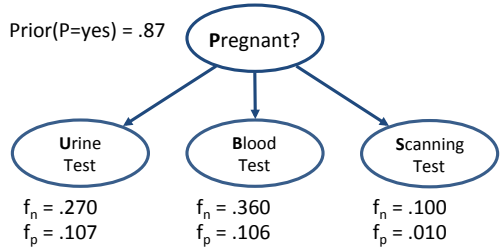}
  \caption{A naive Bayes classifier, specified using the class prior, in addition to the false positive (\(f_p\)) and false negative (\(f_n\)) rates of features.
  The class variable and features are all binary.}
  \label{fig:nbayes}
\end{figure}

\begin{figure}[t]
  \centering
  \includegraphics[width=.4\linewidth,clip=true,angle=0]{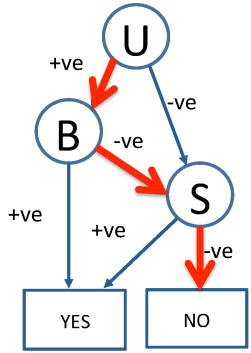}
  \caption{An OBDD (decision function) of the classifier in Figure~\ref{fig:nbayes}.}
  \label{fig:obdd}
\end{figure}

Consider Figure~\ref{fig:nbayes} which depicts a naive Bayes
classifier for detecting pregnancy.  Given results for the three tests, if the probability of
pregnancy passes a given threshold (say \(90\%\)), we would then
obtain a ``yes'' decision on pregnancy. 

Figure~\ref{fig:obdd} depicts the decision function of this classifier,
in the form of an Ordered Binary Decision Diagram (OBDD).
Given some test results, we make a corresponding decision on 
pregnancy by simply navigating the OBDD.  We start at the root, which
is labeled with the Urine (U) test.  Depending on the outcome of this test,
we follow the edge labeled positive, or the edge labeled negative.  We
repeat for the test labeled at the next node.  Eventually, we reach a
leaf node labeled ``yes'' or ``no,'' which provides the resulting
classification. 

The decisions rendered by this OBDD are guaranteed 
to match those obtained from the naive Bayes classifier. 
We have thus converted a probabilistic classifier into an equivalent classifier 
that is symbolic and tractable. 
We will later see how this facilitates the efficient generation of explanations.

We will later discuss compiling Bayesian network classifiers into ODDs,
after formally treating classifiers and ODDs.

\subsection{Bayesian Network Classifiers}

A {\em Bayesian network classifier} is a Bayesian network containing a special set of variables:
a single {\em class} variable \(C\) and \(n\) {\em feature} variables \(\X = \{X_1, \ldots, X_n\}\). 
The class \(C\) is usually a root in the network and the features \(\X\) are usually leaves.
In this paper, we assume that the class variable is binary, with two values \(c\) and \(\bar{c}\)
that correspond to positive and negative classes, respectively (i.e., ``yes'' and ``no'' decisions).
An instantiation of variables \(\X\) is denoted \(\x\) and called an {\em instance.}
A Bayesian network classifier specifying probability distribution \(\pr(.)\) will classify
an instance \(\x\) positively iff \(\pr(c \mid \x) \geq T\), where \(T\) is called the classification {\em threshold.}

\begin{definition}[Decision Function]
Suppose that we have a Bayesian network classifier with features
\(\X\), class variable \(C\) and a threshold \(T\).  Let \(f(\X)\) be a function that maps
instances \(\x\) into \(\{0,1\}\). We say that \(f(\X)\) is the classifier's \underline{decision function} iff
\[
f(\x) =
\left\{ 
\begin{tabular}{cl}
\(1\) & if \(\pr(c \mid \x) \geq T\) \\
\(0\) & otherwise.
\end{tabular}
\right.
\]
Instance \(\x\) is \underline{positive} if \(f(\x)=1\) and \underline{negative} if \(f(\x)=0\).
\end{definition}

The {\em naive Bayes classifier} is a special type of a Bayesian network classifier,
where edges extend from the class to features (no other nodes or edges).
Figure~\ref{fig:nbayes} depicted a naive Bayes classifier.
A {\em latent-tree classifier} is a tree-structured Bayesian network, whose root is the class variable and whose leaves are the features.

\subsection{Monotone Classifiers}

The class of {\em monotone} classifiers is relevant to our discussion, particularly when relating the two types of explanations we shall propose.
We will define these classifiers next, while assuming binary features to simplify the treatment. 
Intuitively, a monotone classifier satisfies the following. A positive instance remains positive if we flip some of its features 
from \(0\) to \(1\). Moreover, a negative instance remains negative if we flip some of its features from \(1\) to \(0\). 

More formally, consider two instances \(\x^\star\) and \(\x\). 
We write \(\x^\star \subseteq^1 \x\) to mean: the features set to \(1\) in \(\x^\star\) is a subset of those set to \(1\) in \(\x\).
Monotone classifiers are then characterized by the following property of their decision functions, which is well-known in the literature on Boolean functions.

\begin{definition}
A decision function \(f(\X)\) is monotone iff 
\[
\x^\star \subseteq^1 \x \mbox{\quad only if \quad} f(\x^\star) \le f(\x).
\]
\end{definition}
One way to read the above formal definition is as follows.
If the positive features in instance \(\x\) contain those in instance \(\x^\star\), then instance \(\x\) must be positive if instance \(\x^\star\) is positive. 

It is generally difficult to decide whether a Bayesian network classifier is monotone; see, e.g.,~\cite{GaagBF04}. However, if the decision function
of the classifier is an OBDD, then monotonicity can be decided in time quadratic in the OBDD size~\cite{HoriyamaI02}. 

\subsection{Ordered Decision Diagrams}

An Ordered Binary Decision Diagram (OBDD) is based on an
ordered set of binary variables \(\X = X_1,\ldots,X_n\). It is a rooted, directed acyclic graph, with
two sinks called the \(1\)-sink and \(0\)-sink.  Every node (except the sinks)
in the OBDD is labeled with a variable \(X_i\) with two outgoing
edges, one labeled \(1\) and the other labeled \(0\).  If
there is an edge from a node labeled \(X_i\) to a node labeled
\(X_j\), then \(i < j\). An OBDD is defined over binary variables,
but can be extended to discrete variables with arbitrary values. This is called an ODD: a node labeled
with variable \(X_i\) has one outgoing edge for each value of variable \(X_i\).
Hence, an OBDD/ODD can be viewed as representing a function \(f(\X)\) that maps instances
\(\x\) into \(\{0,1\}\). Figure~\ref{fig:obdd} depicted an OBDD.
Note: in this paper, we use positive/yes/\(1\) and negative/no/\(0\) interchangeably.

An OBDD is a {\em tractable} representation of a function \(f(\X)\) as it can be used to
efficiently answer many queries about the function. For example, one can in linear time count
the number of positive instances \(\x\) (i.e., \(f(\x)=1\)), called the {\em models} of \(f\). 
One can also conjoin, disjoin and complement OBDDs efficiently. 
This tractability, which carries over to ODDs, will be critical for efficiently generating explanations. 
For more on OBDDs, see \cite{MeinelT98,Wegener00}.

\subsection{Compiling Decision Functions}

\def\extendtree{\texttt{\mbox{expand-then-merge}}}
\def\pickbest{\texttt{\mbox{pick-best-subtree}}}

\cite{chanUAI03} proposed an algorithm for compiling a naive Bayes classifier into an ODD, while guaranteeing
an upper bound on the time of compilation and the size of the resulting ODD. In particular, for a classifier with \(n\) 
features, the compiled ODD has a number of nodes that is bounded by
\(O(b^{\frac{n}{2}})\) and can be obtained in time \(O(n b^{\frac{n}{2}})\). Here, \(b\) is the maximum number
of values that a variable may have. The actual time and space complexity can be much less,
depending on the classifier's parameters and variable order used for the ODD (as observed experimentally).

The algorithm is based on the following insights. Let \(\X\) be all features.
Observing features \(\Y \subset \X\) leads to another naive Bayes classifier, with features \(\X \setminus \Y\) and an adjusted class prior. 
Consider now a decision tree over features \(\X\) and a node in the tree that was reached by a partial instantiation~\(\y\).  
We annotate this node with the corresponding naive Bayes classifier \(N_\y\) found by observing \(\y\),
and then merge nodes with equivalent classifiers---those having equivalent decision functions---as described by \cite{chanUAI03}.  
Implementing this idea carefully leads to an ordered decision diagram (ODD) with the corresponding bounds.\footnote{\cite{chanUAI03}
uses a sophisticated, but conceptually simple, technique for identifying equivalent classifiers.}

\begin{algorithm}[tb]
\caption{\texttt{compile-naive-bayes}(\(N\))} \label{alg:compile-nb}

\Ainput{A naive Bayes classifier \(N\)}

\Aoutput{An ODD for the decision function of \(N\)}

\Amain{
\begin{algorithmic}[1]
\STATE \(D \leftarrow\) empty decision graph
\FOR{each feature \(X\) of classifier \(N\)}
	\STATE \(D \leftarrow \extendtree(N,D,X)\)
\ENDFOR
\RETURN ODD \(D\)
\end{algorithmic}
}
\end{algorithm}

\begin{algorithm}[tb]
\caption{\texttt{compile-latent-tree}(\(N\)) \label{alg:compile-lt}}

\Ainput{A latent-tree classifier \(N\)}

\Aoutput{An ODD for the decision function of \(N\)}

\Amain{
\begin{algorithmic}[1]
\STATE \(D \leftarrow\) empty decision graph
\STATE \(R \leftarrow\) root of tree \(N\)
\WHILE{\(R\) has unprocessed children}
\IF{\(R\) has a single internal and unprocessed child \(C\)}
\STATE \(R \leftarrow C\)
\ELSE
  \STATE \(C \leftarrow\) child of \(R\) with smallest number of leaves
  \FOR{each leaf \(X\) under \(C\)}
	\STATE \(D \leftarrow \extendtree(N,D,X)\)
  \ENDFOR
  \STATE mark \(C\) as processed
\ENDIF
\ENDWHILE
\RETURN ODD \(D\) 
\end{algorithmic}
}
\end{algorithm}

Algorithm~\ref{alg:compile-nb} is a simpler variation on the algorithm of \cite{chanUAI03}; it has the same complexity bounds, but may be less efficient in practice. 
It uses procedure \(\extendtree(.,D,X)\), which expands the partial decision graph \(D\) by a feature \(X\), then merges nodes that correspond
to equivalent classifiers.

Using this procedure, we propose Algorithm~\ref{alg:compile-lt} for compiling a latent-tree classifier into an ODD. Here's the key insight. 
Let \(R\) be a node in a latent-tree classifier where all features outside \(R\) have been observed, and let \(C\) be a child of \(R\).
Observing all features under \(C\) leads to a new latent-tree classifier  
without the subtree rooted at \(C\) and an adjusted class prior. Algorithm~\ref{alg:compile-lt} uses this observation by iteratively
choosing a node \(C\) and then shrinking the classifier size by instantiating the features under \(C\), allowing us to compile an ODD in a fashion similar to \cite{chanUAI03}.
The specific choice of internal nodes \(C\) by Algorithm~\ref{alg:compile-lt} leads to the following complexity.

\begin{theorem}\label{theorem:tree}
Given a latent-tree classifier \(N\) with \(n\) variables, each with at most \(b\) values, 
the ODD computed by Algorithm~\ref{alg:compile-lt} has size \(O(b^{\frac{3n}{4}})\) and can be obtained in time \(O(nb^{\frac{3n}{4}})\).
\end{theorem}

If one makes further assumptions about the structure of the latent tree (e.g., if the root has \(k\) children, and each child of the root has \(O(\frac{n}{k})\) features), 
then one obtains the size bound of \(O(b^\frac{n}{2})\) and time bound of \(O(n b^\frac{n}{2})\) for naive Bayes classifiers.
We do not expect a significantly better upper bound on the time complexity due to the following result.

\begin{theorem}\label{theorem:hardness}
Given a naive Bayes classifier \(N\), compiling an ODD representing its decision function is NP-hard.
\end{theorem}

\section{Minimum Cardinality Explanations} \label{sec:mincard}

We now consider the first type of explanations for why a classifier makes a certain decision. 
These are called {\em minimum-cardinality explanations} or MC-explanations.
We will first assume that the features are binary and then generalize later.

Consider two instances \(\x^\star\) and \(\x\). 
As we did earlier, we write \(\x^\star \subseteq^1 \x\) to mean: the features set to \(1\) in \(\x^\star\) are a subset of those set to \(1\) in \(\x\).
We define \(\x^\star \subseteq^0 \x\) analogously.
Moreover, we write \(\x \leq^1 \x^\star\) to mean: the count of \(1\)-features in \(\x\) is no greater than their count in \(\x^\star\). 
We define \(\x \leq^0 \x^\star\) analogously.

\begin{definition}[MC-Explanation]\label{def:mc-explanation}
Let \(f(\X)\) be a given decision function. An \underline{MC-explanation} of a \underline{positive} instance \(\x\) is another positive instance \(\x^\star\) 
such that \(\x^\star \subseteq^1 \x\) and there is no other positive instance \(\x^\prime \subseteq^1 \x\) where \(\x^\prime <^1 \x^\star\).
An \underline{MC-explanation} of a \underline{negative} instance \(\x\) is another negative instance \(\x^\star\) 
such that \(\x^\star \subseteq^0 \x\) and there is no other negative instance \(\x^\prime \subseteq^0 \x\) where \(\x^\prime <^0 \x^\star\).
\end{definition}
Intuitively, an MC-explanation of a positive decision \(f(\x)=1\) answers the question: 
which positive features of instance \(\x\) are responsible for this decision? 
Similarly for the MC-explanation of a negative decision \(f(\x)=0\): 
which negative features of instance \(\x\) are responsible for this decision?
MC-explanations are not necessarily unique as we shall see later. However, MC-explanations of positive decisions must all have the same number of \(1\)-features,
and those for negative decisions must all have the same number of \(0\)-features.

MC-explanations are perhaps best illustrated using a monotone classifier. 
As a running example, consider a (monotone) classifier for deciding whether 
a student will be admitted to a university. The class variable is \name{admit} ($A$) and the features of an applicant are:
\begin{itemize}
\item \name{work\mbox{-}experience} ($W$): has prior work experience.
\item \name{first\mbox{-}time\mbox{-}applicant} ($F$): did not apply before.
\item \name{entrance\mbox{-}exam} ($E$): passed the entrance exam.
\item \name{gpa} ($G$): has met the university's expected GPA.
\end{itemize}
All variables are either positive (\posve) or negative (\negve).

Consider a naive Bayes classifier with the following false positive and false negative rates:
\begin{center}
\begin{tabular} {c|cc}
feature & \(f_p\) & \(f_n\) \\ \hline
$W$ & \(0.10\) & \(0.04\) \\
$F$ & \(0.20\) & \(0.30\) \\
$E$ & \(0.15\) & \(0.60\) \\
$G$ & \(0.11\) & \(0.03\)
\end{tabular}
\end{center}
To completely specify the naive Bayes classifier, we also need the prior probability
of admission, which we assume to be \(\pr(\eql(A,\posve)) = 0.30\). Moreover, we
use a decision threshold of \(0.50\), admitting an applicant \(\x\) if \(\pr(\eql(A,\posve)\mid\x) \geq .50\).  Note that with the above false positive and false negative rates, a positively observed feature will increase the probability of a positive classification, while a negatively observed feature will increase the probability of a negative classification (hence, the classifier is monotone).

\begin{table}[tb]
\begin{center}
\small
\setlength{\tabcolsep}{3pt}
\begin{tabular}{cccc|c|c|c}
$W$ & $F$ & $E$ & $G$ & \(\pr(\eql(A,\posve) | \x)\) & \(f(\x)\) & MC-explanations\\\hline
\negve & \negve & \negve & \negve & 0.0002 & \negve & \parbox{3.2cm}{\centering (\negve\ \negve\ \posve\ \posve) (\negve\ \posve\ \negve\ \posve) (\negve\ \posve\ \posve\ \negve)\\ (\posve\ \negve\ \posve\ \negve)  (\posve\ \posve\ \negve\ \negve)} \\
\negve & \negve & \negve & \posve & 0.0426 & \negve & (\negve\ \negve\ \posve\ \posve) (\negve\ \posve\ \negve\ \posve) \\
\negve & \negve & \posve & \negve & 0.0006 & \negve & (\negve\ \negve\ \posve\ \posve) (\negve\ \posve\ \posve\ \negve) (\posve\ \negve\ \posve\ \negve)\\
\negve & \negve & \posve & \posve & 0.1438 & \negve & (\negve\ \negve\ \posve\ \posve) \\
\negve & \posve & \negve & \negve & 0.0016 & \negve & (\negve\ \posve\ \negve\ \posve) (\negve\ \posve\ \posve\ \negve) (\posve\ \posve\ \negve\ \negve) \\
\negve & \posve & \negve & \posve & 0.2933 & \negve & (\negve\ \posve\ \negve\ \posve) \\
\negve & \posve & \posve & \negve & 0.0060 & \negve & (\negve\ \posve\ \posve\ \negve) \\
\negve & \posve & \posve & \posve & 0.6105 & \color{red}\textbf{\posve} & \color{red}\textbf{(\negve\ \posve\ \posve\ \posve)}\\
\posve & \negve & \negve & \negve & 0.0354 & \negve & (\posve\ \posve\ \negve\ \negve) (\posve\ \negve\ \posve\ \negve) \\
\posve & \negve & \negve & \posve & 0.9057 & \color{red}\textbf{\posve} & \color{red}\textbf{(\posve\ \negve\ \negve\ \posve)}\\
\posve & \negve & \posve & \negve & 0.1218 & \negve & (\posve\ \negve\ \posve\ \negve) \\
\posve & \negve & \posve & \posve & 0.9732 & \color{red}\textbf{\posve} & \color{red}\textbf{(\posve\ \negve\ \negve\ \posve)}\\
\posve & \posve & \negve & \negve & 0.2552 & \negve & (\posve\ \posve\ \negve\ \negve) \\
\posve & \posve & \negve & \posve & 0.9890 & \color{red}\textbf{\posve} & \color{red}\textbf{(\posve\ \negve\ \negve\ \posve)}\\
\posve & \posve & \posve & \negve & 0.5642 & \color{red}\textbf{\posve} & \color{red}\textbf{(\posve\ \posve\ \posve\ \negve)}\\
\posve & \posve & \posve & \posve & 0.9971 & \color{red}\textbf{\posve} & \color{red}\textbf{(\posve\ \negve\ \negve\ \posve)}\\
\end{tabular}
\end{center}
\caption{A decision function with MC-explanations. \label{tab:decision-fn}}
\end{table}

Table~\ref{tab:decision-fn} depicts the decision function \(f\) for this naive Bayes classifier,
with MC-explanations for all \(16\) instances.

Consider, for example, a student (\posve\ \posve\ \posve\ \posve) who was admitted by this decision function. 
There is a single MC-explanation for this decision, (\posve\ \negve\ \negve\ \posve), with cardinality \(2\). According to this explanation, 
work experience and a good GPA were the reasons for admission. That is, the student would still have been admitted
even if they have applied before and did not pass the entrance exam.

For another example, consider a student (\negve\ \negve\ \negve\ \posve) who was rejected. 
There are two MC-explanations for this decision. The first, (\negve\ \negve\ \posve\ \posve), says that the student would not have been
admitted, even if they passed the entrance exam. The second explanation, (\negve\ \posve\ \negve\ \posve), says that the student would not have been admitted,
even if they were a first-time applicant.

Finally, we remark that while MC-explanations are more intuitive for monotone classifiers, they also apply to classifiers that are not monotone, as we shall see in Section~\ref{sec:case-study}. 

\subsection{Computing MC-Explanations}

We will now present an efficient algorithm for computing the MC-explanations of a decision, assuming that the decision function
has a specific form. Our treatment assumes that the decision function is represented as an OBDD, but it actually applies to a broader class of
representations which includes OBDDs as a special case. More on this later.

Our algorithm uses a key operation on decision functions. 

\begin{definition}[Cardinality Minimization]
For \(i \in \{0,1\}\), the \underline{\(i\)-minimization} of decision function \(f(\X)\) is 
another decision function \(f^i(\X)\) defined as follows:
\(f^i(\x)=1\) iff 
(a)~\(f(\x)=1\) and 
(b)~\(\x \leq^i \x^\star\) for every \(f(\x^\star)=1\).
\end{definition}
The \(1\)-minimization of decision function \(f\) renders positive decisions only on the positive instances of \(f\) having a minimal number of \(1\)-features.
Similarly, the \(0\)-minimization of decision function \(f\) renders positive decisions only on the positive instances of \(f\) having a minimal number of \(0\)-features.
Cardinality minimization was discussed and employed for other purposes in \cite{darwicheJACM-DNNF,ChoiKisaDarwiche13}.

\begin{algorithm}[tb]
\caption{\texttt{find-mc-explanation}(\(f(\X),\x\))} \label{alg:find-mc}

\Ainput{An OBDD \(f(\X)\) and instance \(\x\).}

\Aoutput{An OBDD \(g(\X)\) where \(g(\x^\star)=1\) iff \(\x^\star\) is an MC-explanation of decision \(f(\x)\).}

\Amain{
\begin{algorithmic}[1]
\STATE \(i \leftarrow f(\x)\)
\STATE \(\alpha \leftarrow \mbox{the subset of \(\x\) with variables set to \(1-i\)}\)
\STATE complement function \(f\) if \(i=0\)
\RETURN \(i\)-\(\mincard(\conjoin(f,\alpha))\)
\end{algorithmic}
}
\end{algorithm}

Algorithm~\ref{alg:find-mc} computes the MC-explanations of a decision \(f(\x)\).
The set of computed explanations is encoded by another decision function \(g(\X)\). In particular, \(g(\x^\star)=1\) iff
\(\x^\star\) is an MC-explanation of decision \(f(\x)\).

Suppose we want to compute
the MC-explanations of a positive decision \(f(\x)=1\). The algorithm will first find the portion \(\alpha\) of instance \(\x\) with variables
set to \(0\). It will then conjoin\footnote{Conjoining \(f\) with \(\alpha\) leads to a function \(h\) such that
\(h(\x)=1\) iff \(f(\x)=1\) and \(\x\) is compatible with \(\alpha\).} \(f\) with \(\alpha\) and \(1\)-minimize the result. 
The obtained decision function encodes the MC-explanations in this case. 

An OBDD can be complemented and conjoined with a variable instantiation in linear time. It can also be
minimized in linear time. This leads to the following complexity for generating MC-explanations based on OBDDs.

\begin{theorem}\label{theo:mc-explanation}
When the decision function \(f(\X)\) is represented as an OBDD, the time and space complexity of Algorithm~\ref{alg:find-mc} is linear in the size
of \(f\), while guaranteeing that the output function \(g(\X)\) is also an OBDD.
\end{theorem}
Given OBDD properties, one can count MC-explanations in linear time, and enumerate each in linear time.\footnote{Minimization, 
conjoin, and model enumeration are all linear time operations on DNNFs, which is a superset of OBDDs \cite{darwicheJACM-DNNF,darwicheJAIR02}. 
Moreover,
\[
\mbox{OBDD} \subset \mbox{SDD} \subset \mbox{d-DNNF} \subset \mbox{DNNF}
\]
where we read \(\subset\) as ``is-a-subclass-of''. 
Hence, DNNFs, d-DNNFs and SDDs could have been used for supporting
MC-explanations, except that we would need a different algorithm for compiling classifiers.
Moreover, beyond OBDDs, only SDDs support complementation in linear time. Hence, efficiently
computing MC-explanations of negative decision requires that we efficiently complement the decision functions represented by DNNFs or d-DNNFs.
}  

\subsection{Case Study: Votes Classifier} \label{sec:case-study}

We now consider the Congressional Voting Records (\name{votes}) from the UCI machine learning repository \cite{BacheLichman2013}.  This dataset consists of 16 key votes by Congressmen of the U.S. House of Representatives.  The class label is the party of the Congressman (positive if Republican and negative if Democrat).
A naive Bayes classifier trained on this dataset obtains 91.0\% accuracy.  We compiled this classifier into an OBDD, which has a size of 630 nodes.

The following Congressman from the dataset voted on all 16 issues and was classified correctly as a Republican:
\begin{center}
(0 1 0 1 1 1 0 0 0 0 0 0 1 1 0 1)
\end{center}
This decision has five MC-explanations of cardinality 3, e.g.:
\begin{center}
(0 0 0 1 1 0 0 0 0 0 0 0 0 1 0 0)
\end{center}
The MC-explanation tells us that this Congressmen could have reversed four of their yes-votes, and the classifier would still predict that this Congressman was a Republican.

For a problem of this size, we can enumerate all instances of the classifier.  We computed the MC-explanations for each of the \(32,256\) positive instances, 
out of a possible number of \(2^{16} = 65,536\) instances.  Among these MC-explanations, the one that appeared the most frequently was the MC-explanation from the above example.
This explanation corresponded to yes-votes on three issues:
\name{physician\mbox{-}fee\mbox{-}freeze},
\name{el\mbox{-}salvador\mbox{-}aid}, and 
\name{crime}.
Further examination of the dataset revealed that these issues were the three with the fewest Republican no-votes.

\section{Prime Implicant Explanations} \label{sec:primes}

We now consider the second type of explanations, called
\emph{prime-implicant explanations} or PI-explanations for short.

Let \(\y\) and \(\z\) be instantiations of some features and call them {\em partial instances.}
We will write \(\y \supseteq \z\) to mean that \(\y\) extends \(\z\), that is, it includes \(\z\) but
may set some additional features.

\begin{definition}[PI-Explanation]\label{def:pi-explanation}
Let \(f(\X)\) be a given decision function.  A \underline{PI-explanation} of a decision \(f(\x)\) is a partial instance \(\z\)
such that
\begin{enumerate}
\item[(a)] \(\z \subseteq \x\),
\item[(b)] \(f(\x) = f(\x^\star)\) for every \(\x^\star \supseteq \z\), and
\item[(c)] no other partial instance \(\y \subset \z\) satisfies (a) and (b).
\end{enumerate}
\end{definition}
Intuitively, a PI-explanation of decision \(f(\x)\) is a minimal subset \(\z\) of instance \(\x\)
that makes features outside \(\z\) irrelevant to the decision. That is, we can toggle any feature that does
not appear in \(\z\) while maintaining the current decision. The number of features appearing in a PI-explanation
will be called the {\em length} of the explanation. As we shall see later, PI-explanations of the same decision may have
different lengths.

\begin{table}[tb]
\begin{center}
\small
\setlength{\tabcolsep}{3pt}
\begin{tabular}{cccc|c|c|c}
$W$ & $F$ & $E$ & $G$ & \(\pr(\eql(A,\posve) | \x)\) & \(f(\x)\) & PI-explanations\\\hline
\rule{0pt}{9pt}\negve & \negve & \negve & \negve & 0.0002 & \negve & ($\n(w)\n(f)$) ($\n(w)\n(e)$) ($\n(w)\n(g)$) ($\n(f)\n(g)$) ($\n(e)\n(g)$) \\
\negve & \negve & \negve & \posve & 0.0426 & \negve & ($\n(w)\n(f)$) ($\n(w)\n(e)$) \\
\negve & \negve & \posve & \negve & 0.0006 & \negve & ($\n(w)\n(f)$) ($\n(w)\n(g)$) ($\n(f)\n(g)$) \\
\negve & \negve & \posve & \posve & 0.1438 & \negve & ($\n(w)\n(f)$) \\
\negve & \posve & \negve & \negve & 0.0016 & \negve & ($\n(w)\n(e)$) ($\n(w)\n(g)$) ($\n(e)\n(g)$) \\
\negve & \posve & \negve & \posve & 0.2933 & \negve & ($\n(w)\n(e)$) \\
\negve & \posve & \posve & \negve & 0.0060 & \negve & ($\n(w)\n(g)$) \\
\negve & \posve & \posve & \posve & 0.6105 & \color{red}\textbf{\posve} & \color{red}\textbf{($feg$)} \\
\posve & \negve & \negve & \negve & 0.0354 & \negve & ($\n(f)\n(g)$) ($\n(e)\n(g)$) \\
\posve & \negve & \negve & \posve & 0.9057 & \color{red}\textbf{\posve} & \color{red}\textbf{($wg$)} \\
\posve & \negve & \posve & \negve & 0.1218 & \negve & ($\n(f)\n(g)$) \\
\posve & \negve & \posve & \posve & 0.9732 & \color{red}\textbf{\posve} & \color{red}\textbf{($wg$)} \\
\posve & \posve & \negve & \negve & 0.2552 & \negve & ($\n(e)\n(g)$) \\
\posve & \posve & \negve & \posve & 0.9890 & \color{red}\textbf{\posve} & \color{red}\textbf{($wg$)} \\
\posve & \posve & \posve & \negve & 0.5642 & \color{red}\textbf{\posve} & \color{red}\textbf{($wfe$)} \\
\posve & \posve & \posve & \posve & 0.9971 & \color{red}\textbf{\posve} & \color{red}\textbf{($wg$) ($wfe$) ($feg$)} \\
\end{tabular}
\end{center}
\caption{A decision function with PI-explanations. \label{tab:decision-fn-pi}}
\end{table}

Table~\ref{tab:decision-fn-pi}  depicts the decision function
\(f\) for the admissions classifier, with PI-explanations
for all 16 instances.  We write \((wg)\) for \(\eql(W,\posve), \eql(G,\posve)\)
and ($\n(e)\n(g)$) for \(\eql(E,\negve), \eql(G,\negve)\).

Consider a student (\posve\ \posve\ \negve\ \negve) who was \emph{not} admitted by this decision function. 
There is a single PI-explanation ($\n(e)\n(g)$) for this decision.  According to this explanation, it is sufficient to have a poor entrance exam and a poor GPA to be rejected---it does not matter whether they have work experience or if they are a first-time applicant.  
That is, we can set these features to any value, and the applicant would still be rejected.

Consider now a student (\posve\ \posve\ \posve\ \posve) who was admitted. 
There are three PI-explanations for this decision, ($wg$) ($wfe$) ($feg$), with different lengths.
These explanations can be visualized as
(\posve\ * * \posve), (\posve\ \posve\ \posve\ *) and (*\ \posve\ \posve\ \posve).
This is in contrast to the single MC-explanation (\posve\ \negve\ \negve\ \posve) obtained previously.

\subsection{Computing Prime Implicant Explanations}

\def\picover{\texttt{\mbox{pi-cover}}}
\def\piinst{\texttt{\mbox{pi-inst}}}

Algorithms exist for converting an OBDD for function \(f\) into an ODD that 
encodes the prime implicants of \(f\) \cite{CoudertMadre93,PrimeOverview93,minato1993fast}.\footnote{These algorithms compute prime-implicant {\em covers}.}
The resulting ODD has three values for each variable: \(0\), \(1\) and \(\ast\) (don't care). The ODD encodes partial instances, which correspond to the PI-explanations
of positive instances (to get the PI-explanations of negative instances, we complement the OBDD \(f\)).
These algorithms recurse on the structure of the input OBDD, computing prime implicants of sub-OBDDs. 
If \(X\) is the variable labeling the root of OBDD \(f\), then \(f_{\n(x)}\) denotes its \(0\)-child and \(f_x\) denotes its \(1\)-child.  
Algorithm~\ref{alg:pi-cover} computes prime implicants by recursively computing prime implicants for \(f_{\n(x)}\), \(f_x\) and \(f_{\n(x)} \wedge f_x\)  \cite{CoudertMadre93}.  

As we are interested in explaining a specific instance \(\x\), we only need the prime implicants compatible with \(\x\) (a  
function may have exponentially many prime implicants, but those compatible with an instance may be small).
We exploit this observation in Algorithm~\ref{alg:pi-inst}, which computes the PI-explanations of a given positive instance \(\x\) by
avoiding certain recursive calls.  
Empirically, we have observed that Algorithm~\ref{alg:pi-inst} can be twice as fast as Algorithm~\ref{alg:pi-cover} (computing PIs first, then conjoining with a given instance to obtain PI-explanations).
It can also generate ODDs that are an order-of-magnitude smaller.
The following table highlights this difference in size and running time, per instance, between Algorithms~\ref{alg:pi-cover}~(cover) \&~\ref{alg:pi-inst}~(inst).  Relative improvements are denoted by impr; \(n\) denotes the number of features.
We report averages over \(50\) instances.

\begin{algorithm}[tb]
\caption{\picover(\(f,\pi\))} \label{alg:pi-cover}

\Ainput{OBDD \(f\) and variable ordering \(\pi\)}

\Aoutput{ODD \(g\) encoding prime implicants of \(f\)}

\Amain{
\begin{algorithmic}[1]
\STATE \textbf{if} \(\pi\) is empty \textbf{return} \(f\)
\STATE remove first variable \(X\) from order \(\pi\)
\STATE \(g_\ast \leftarrow \picover(f_{\n(x)} \wedge f_{x},\pi)\)
\STATE \(g_{\n(x)} \leftarrow \picover(f_{\n(x)},\pi),\:\:\) \(g_{x} \leftarrow \picover(f_x,\pi)\)
\STATE \(g_{\n(x)} \leftarrow g_{\n(x)} \wedge \neg g_\ast,\:\:\) \(g_{x}  \leftarrow g_{x} \wedge \neg g_\ast\)
\RETURN ODD with branches \(g_{\n(x)}, g_{x}, g_\ast\)
\end{algorithmic}
}
\end{algorithm}

\begin{algorithm}[tb]
\caption{\piinst(\(f,\pi,\x\))} \label{alg:pi-inst}

\Ainput{OBDD \(f\), variable ordering \(\pi\), and instance \(\x\)}

\Aoutput{ODD \(g\) for primes implicant compatible with \(\x\)}

\Amain{
\begin{algorithmic}[1]
\STATE \textbf{if} \(\pi\) is empty \textbf{return} \(f\)
\STATE remove first variable \(X\) from order \(\pi\)
\STATE \(g_\ast \leftarrow \piinst(f_{\n(x)} \wedge f_x,\pi,\x)\)
\IF{\(\x\) sets \(X\) to \(\n(x)\)}
	\STATE \(g_{\n(x)} \leftarrow \piinst(f_{\n(x)},\pi,\x),\:\:\) \(g_x \leftarrow\bot\)
\ELSE
	\STATE \(g_{\n(x)} \leftarrow \bot,\:\:\) \(g_{x} \leftarrow \piinst(f_x,\pi,\x)\)
\ENDIF
\STATE \(g_{\n(x)} \leftarrow g_{\n(x)} \wedge \neg g_\ast,\:\:\) \(g_{x} \leftarrow g_{x} \wedge \neg g_\ast\)
\RETURN ODD with branches \(g_{\n(x)}, g_{x}, g_\ast\)
\end{algorithmic}
}
\end{algorithm}

\begin{center}
\setlength{\tabcolsep}{6pt}
\small
\begin{tabular}{rr|rrr|rrr}
& & \multicolumn{3}{|c|}{time (s)} & \multicolumn{3}{|c}{ODD size} \\ \hline
dataset & $n$ & cover & inst & impr & cover & inst & impr \\ \hline
\name{votes} & 16 & 0.04 & 0.02 & 1.99 & 2,144 & 139 & 15.42 \\
\name{spect} & 22 & 0.06 & 0.02 & 2.27 & 3,130 & 437 & 7.14 \\
\name{msnbc} & 16 & 0.07 & 0.02 & 2.56 & 5,086 & 446 & 11.39\\
\name{nltcs} & 15 & 0.03 & 0.02 & 1.39 & 432 & 111 & 3.89
\end{tabular}
\end{center}

\subsection{Case Study: Votes Classifier}

Consider again the voting record of the Republican Congressman that we considered earlier in Section~\ref{sec:case-study}:
\begin{center}
(0 1 0 1 1 1 0 0 0 0 0 0 1 1 0 1)
\end{center}
There are 30 PI-explanations of this decision.  There are 2 shortest explanations of \(9\) features:
\begin{center}
(\hphantom{0} \hphantom{0} 0 1 1 \hphantom{0} 0 0 0 \hphantom{0} \hphantom{0} \hphantom{0} 1 1 0 \hphantom{0})\\
(\hphantom{0} \hphantom{0} 0 1 1 1 \hphantom{0} 0 0 \hphantom{0} \hphantom{0} \hphantom{0} 1 1 0 \hphantom{0})
\end{center}
The first corresponds to yes votes on:
\begin{center}
\name{physician\mbox{-}fee\mbox{-}freeze},
\name{el\mbox{-}salvador\mbox{-}aid}, 
\name{superfund\mbox{-}right\mbox{-}to\mbox{-}sue},
\name{crime},
\end{center}
and no votes on
\begin{center}
\name{adoption\mbox{-}of\mbox{-}the\mbox{-}budget\mbox{-}resolution},
\name{anti\mbox{-}satellite\mbox{-}test\mbox{-}ban},
\name{aid\mbox{-}to\mbox{-}nicaraguan\mbox{-}contras},
\name{mx\mbox{-}missile},
\name{duty\mbox{-}free\mbox{-}exports}.
\end{center}
These 9 votes necessitate the classification of a Republican; no other vote changes this decision.
Finally, there are \(506\) PI-explanations for all decisions made by this classifier:
\begin{center}
\setlength{\tabcolsep}{5pt}
\begin{tabular}{c|ccccc|c}
length of explanation & 9 & 10 & 11 & 12 & 13 & total \\\hline
number of explanations & 35 & 308 & 143 & 19 & 1 & 506
\end{tabular}
\end{center}

\section{More On Monotone Classifiers} \label{sec:more-monotone}

We now discuss a specific relationship between MC and PI explanations for monotone classifiers.

An MC-explanation sets all features, while a PI-explanation sets only a subset of the features. 
For a positive instance, we will say that MC-explanation \(\x\) and PI-explanation \(\z\) {\em match} 
iff \(\x\) can be obtained from \(\z\) by setting all missing features negatively. 
For a negative instance, MC-explanation \(\x\) and PI-explanation \(\z\) match 
iff \(\x\) can be obtained from \(\z\) by setting all missing features positively. 

\begin{theorem}\label{theorem:monotone}
For a decision \(f(\x)\) of a monotone decision function \(f\):
\begin{enumerate}
\item Each MC-explanation matches some shortest PI-explanation.
\item Each shortest PI-explanation matches some MC-explanation.
\end{enumerate}
\end{theorem}
Hence, for monotone decision functions, MC-explanations coincide with shortest PI-explanations.

The admissions classifier we considered earlier is monotone, which can be verified by inspecting its decision function (in contrast, the votes classifier is not monotone).
Here, all MC-explanations matched PI-explanations. 
For example, the MC-explanation (\posve\ \negve\ \negve\ \posve) for instance (\posve\ \posve\ \negve\ \posve) matches the PI-explanation \((wg)\).
However, the PI-explanation \((wfe)\) for instance 
(\posve\ \posve\ \posve\ \posve) does not match the single MC-explanation (\posve\ \negve\ \negve\ \posve). One can
verify though, by examining Tables~\ref{tab:decision-fn} and~\ref{tab:decision-fn-pi}, that shortest PI-explanations coincide with MC-explanations.

MC-explanations are no longer than PI-explanations and their count is no larger than the count of PI-explanations. 
Moreover, MC-explanations can be computed in linear time, given that the decision function is represented as an OBDD.
This is not guaranteed for PI-explanations.  

PI-explanations can be directly extended to classifiers with multi-valued features. They are also meaningful for arbitrary classifiers, not just monotone ones.
While our definition of MC-explanations was directed towards monotone classifiers with binary features, it can be generalized so it remains useful for arbitrary
classifiers with multi-valued features. In particular, let us partition the {\em values} of each feature into two sets: on-values and off-values. Let us also 
partition the set of {\em features} \(\X\) into \(\Y\) and \(\Z\). Consider now the following question about a decision \(f(\x)\), where \(\x = \y\z\). Keeping \(\y\) fixed, find
a culprit of on-features in \(\z\) that maintains the current decision. Definition~\ref{def:mc-explanation} is a special case of this more general definition,
and Algorithm~\ref{alg:find-mc} can be easily extended to compute these more general MC-explanations using the same complexity (that is, linear in the size of ODD for the
decision function).

\section{Related Work} \label{sec:related}

There has been significant interest recently in providing explanations for classifiers; see, e.g., \cite{lime:kdd16,ElenbergDFK17,LundbergL17,anchor:nipsws16,anchors:aaai18}.  
In particular, \emph{model-agnostic} explainers were sought \cite{lime:kdd16}, which can explain the behavior of (most) any classifier, by treating it as a \emph{black box}.  Take for example, LIME, which \emph{locally} explains the classification of a given instance.  Roughly, LIME samples new instances that are ``close'' to a given instance, and then learns a simpler, interpretable model from the sampled data.  For example, suppose a classifier rejects a loan to an applicant; one could learn a decision tree for other instances similar to the applicant, to understand why the original decision was made.

More related to our work is the notion of an ``anchor'' introduced in \cite{anchor:nipsws16,anchors:aaai18}.  
An anchor for an instance is a subset of the instance that is highly likely to be classified with the same label, no matter how the missing features are filled in (according to some distribution).  An anchor can be viewed as a probabilistic extension of a PI-explanation.  Anchors can also be understood using the Same-Decision Probability (SDP) \cite{ChoiXueDarwiche11,ChenJAIR14,ChoiIJCAI17}, proposed in \cite{DarwicheChoi10}.  In this context, the SDP asks, ``Given that I have already observed \(\x\), what is the probability that I will make the same classification if I observe the remaining features?''  In this case, we expect an anchor \(\x\) 
to have a high SDP, but a PI-explanation \(\x\) will always have an SDP of 1.0.

\section{Conclusion} \label{sec:conclusion}

We proposed an algorithm for compiling latent-tree Bayesian network
classifiers into decision functions in the form of ODDs. We also
proposed two approaches for explaining the decision that a Bayesian
network classifier makes on a given instance, which apply more
generally to any decision function in symbolic form. One approach is
based on MC-explanations, which minimize the number of positive
features in an instance, while maintaining its classification.  The
other approach is based on PI-explanations, which identify a smallest
set of features in an instance that renders the remaining features
irrelevant to a classification.  We proposed algorithms for computing
these explanations when the decision function has a symbolic and
tractable form. We also discussed monotone classifiers and showed that
MC-explanations and PI-explanations coincide for this class of
classifiers.

\section*{Acknowledgments}

This work has been partially supported by NSF grant \#IIS-1514253,
ONR grant \#N00014-15-1-2339 and DARPA XAI grant \#N66001-17-2-4032.

\appendix

\section{Proofs}

\begin{proof}[Proof of Theorem~\ref{theorem:tree}]
Our proof is based on analyzing Algorithm \ref{alg:compile-lt} on an arbitrary latent-tree classifier with \(n\) variables and \(b\) values, and bounding the size of the decision graph \(D\) after each call to \(\extendtree\). For any iteration of the while-loop, let \(D\) be the initial decision graph and let \(D'\) be the decision graph generated after the expanding phase of \(\extendtree(.,D,.)\). Furthermore, let \(S(D)\) denote the number of leaf nodes of \(D\) (similarly for \(D'\)). We will show the following loop invariant: \(S(D') \leq b^{\frac{3n}{4}}\). For any iteration, \(S(D)\) is bounded by \(\min(b^i,b^{n-i})\), where \(i\) denotes the depth of \(D\).
There are \(n-i\) variables remaining, and the choice of \(C\) in the algorithm guarantees that the number of variables under \(C\) is at most \(\frac{n-i}{2}\). Thus, \(S(D')\) is bounded by \(b^{\frac{n-i}{2}}S(D) = \min(b^{\frac{n+i}{2}},b^{\frac{3(n-i)}{2}})\). If \(i \leq \frac{n}{2}\), then \(S(D') \leq b^{\frac{n+n/2}{2}} = b^\frac{3n}{4}\). Otherwise if \(i > \frac{n}{2}\) then \(S(D') \leq b^\frac{3(n-n/2)}{2} = b^\frac{3n}{4}\).
Thus, after every call to \(\extendtree\), the decision graph \(D'\) has at most \(b^\frac{3n}{4}\) leaf nodes and the merging phase cannot increase the number of nodes, giving us a total size bound of \(O(nb^\frac{3n}{4})\). To obtain the size bound of \(O(b^\frac{3n}{4})\), observe that \(S(D')\) is at least half of the number of newly expanded nodes for each call, and at most one such call can have \(S(D') > b^\frac{2n}{3}\) nodes. Finally, merging a node in \(D'\) takes time logarithmic in the size of \(D'\), so the time complexity is \(O(nb^\frac{3n}{4})\).
\end{proof}

\begin{proof}[Proof of Theorem~\ref{theorem:hardness}]
Our proof is based on \cite{ChenJAIR14}, which showed that computing the same-decision probability (SDP) is NP-hard in naive Bayes networks.  Say we have an instance of the number partitioning problem, where we have positive integers \(a_1,\ldots,a_n\) and we ask if there exists a set \(I \subseteq \{1,\ldots,n\}\) such that \(\sum_{i \in I} a_i = \sum_{i \notin I} a_i\).  Suppose we have a naive Bayes classifier with features \(X_i\) where:
\[
\log \frac{\pr(x_i \mid c)}{\pr(x_i \mid \n(c))} = a_i
\mbox{\quad and \quad}
\log \frac{\pr(\n(x)_i \mid c)}{\pr(\n(x)_i \mid \n(c))} = -a_i
\]
and where we have a uniform prior \(\pr(C)\).  Let \(\x_I\) be the instance where \(X_i\) is set to true if \(i \in I\) and \(X_i\) is set to false if \(i \notin I\) .
Consider the log-odds \(\log O(c \mid \x_I) = \log \frac{\pr(c \mid \x_I)}{\pr(\n(c) \mid \x_I)}\):
\begin{align*}
\log O(c \mid \x_I) 
& = \sum_{i \in I} \log \frac{\pr(x_i \mid c)}{\pr(x_i \mid \n(c))}
  + \sum_{i \notin I} \log \frac{\pr(\n(x)_i \mid c)}{\pr(\n(x)_i \mid \n(c))} \\
& = \left( \sum_{i \in I} a_i \right) - \left( \sum_{i \notin I} a_i \right)
\end{align*}
If \(I\) is a number partitioning solution, then \(\log O(c \mid \x_I) = 0.\)  Otherwise \(\log O(c \mid \x_I) = - \log O(c \mid \x_J) \ne 0\) where \(J = \{1,\ldots,n\} \setminus I\).  Hence, if there is no solution \(I\), then half of the instances \(\x\) have log-odds strictly greater than zero, and the other half have log-odds strictly less than zero.  Thus, there exists a solution iff the number of positive instances in the decision function of \(N\) is strictly less than \(\frac{1}{2} \cdot 2^n\) given a (strict) threshold of \(\frac{1}{2}\).
Finally, if we can compile the decision function of \(N\) to an OBDD in polytime, then we can perform model counting in time linear in the size of the OBDD, and hence solve number partitioning, which is NP-complete.  Thus, compiling the decision function is NP-hard.
\end{proof}

\begin{proof}[Proof of Theorem~\ref{theo:mc-explanation}]
An OBDD \(f\) can be complemented by simply switching its \(0\)-sink and \(1\)-sink.
Since \(\alpha\) is a conjunction of literals, we can conjoin \(f\) with \(\alpha\) by manipulating the OBDD structure directly: if \(X\) appears in \(\alpha\) positively (negatively), we redirect the 0-edge (1-edge) of each OBDD node labeled by \(X\) to the \(0\)-sink.  Clearly, this operation takes time linear in the size of \(f\).  The operation of \(i\)-minimization can also be performed in time linear in the size of \(f\) using the technique given in \cite{darwicheJACM-DNNF} for DNNFs. The minimization procedure
performs two passes. The first pass performs an addition or minimization at each node. The second pass redirects some edges depending on simple tests.
\end{proof}

\begin{proofendline}[Proof of Theorem~\ref{theorem:monotone}]
Suppose, without loss of generality, that we are explaining a positive instance \(\x^\star\) of a monotone decision function \(f\) (the negative case is symmetric).
The proof uses the following observation: A shortest PI-explanation \(\z\) must have all its features set positively (otherwise, due to monotonicity, we can just drop the negative features in \(\z\) to obtain a shorter PI-explanation).
\begin{enumerate}
\item Suppose that \(\x\) is an MC-explanation. Let \(\z\) be the portion of \(\x\) containing all features that are set positively. 
Due to monotonicity, we can toggle features of \(\x\) that are outside \(\Z\) without changing the decision.
Moreover, no subset of \(\z\) will have this property; otherwise, \(\x\) cannot be an MC-explanation.
Hence, \(\z\) is a PI-explanation that matches \(\x\). 
Suppose now that \(\z\) is not a shortest PI-explanation and let  \(\z^\prime\) be a shortest PI-explanation.  Then we can augment \(\z^\prime\) by setting all missing features negatively, giving us a positive instance with a 1-cardinality less than that of \(\x\).  Hence, \(\x\) cannot be an MC-explanation.

\item Suppose that \(\z\) is a shortest PI-explanation. Then all features in \(\z\) must be set positively.
Now let \(\x\) be the result of augmenting \(\z\) by setting all missing features negatively. Then \(\x\) is a positive instance since \(\z\) is a PI-explanation. 
Suppose now that \(\x\) is not an MC-explanation, and let \(\x^\prime\) be an MC-explanation.  Then let \(\z^\prime\) be the portion of \(\x^\prime\) containing all features that are set positively.  By monotonicity, \(\z\) cannot be a shortest PI-explanation since \(\z^\prime\) is shorter than \(\z\) yet all of its completions would be positive instances.
\end{enumerate}
\end{proofendline}

{\small
\bibliographystyle{named}
\bibliography{bib/references}
}

\end{document}